\documentclass[sigconf]{acmart}
\usepackage{enumitem}
\usepackage{multirow}
\usepackage{makecell}
\usepackage{color}
\usepackage{graphicx}
\usepackage[utf8]{inputenc} 
\usepackage[T1]{fontenc}    
\usepackage{hyperref}       
\usepackage{url}            
\usepackage{booktabs}       
\usepackage{amsfonts}       
\usepackage{nicefrac}       
\usepackage{microtype}      
\usepackage{xcolor}         
\usepackage{marvosym}
\usepackage{colortbl}
\usepackage{tcolorbox}
\usepackage{longtable}
\usepackage{subcaption}
\usepackage{pifont}
\usepackage{framed}
\usepackage[utf8]{inputenc}
\usepackage{longtable}
\definecolor{mypink}{rgb}{.99,.91,.95}
\definecolor{mygreen}{rgb}{.9,.99,.9}
\definecolor{myorange}{RGB}{255, 217, 192}
\definecolor{mygray}{gray}{.9}
\definecolor{shadecolor}{rgb}{0.9,0.9,0.9} 
\AtBeginDocument{%
  }

\copyrightyear{2026}
\acmYear{2026}
\setcopyright{cc}
\setcctype{by}
\acmConference[KDD '26]{Proceedings of the 32nd ACM SIGKDD Conference on Knowledge Discovery and Data Mining V.2}{August 09--13, 2026}{Jeju Island, Republic of Korea}
\acmBooktitle{Proceedings of the 32nd ACM SIGKDD Conference on Knowledge Discovery and Data Mining V.2 (KDD '26), August 09--13, 2026, Jeju Island, Republic of Korea}
\acmDOI{10.1145/3770855.3818367}
\acmISBN{979-8-4007-2259-2/2026/08}

\begin{document}

\title{Denoising Implicit Feedback for Cold-start Recommendation}

\author{Gaode Chen}
\authornote{Equal contributions.}
\affiliation{%
  \institution{Kuaishou Technology}
  \city{Beijing}
  \country{China}
  }
\email{chengaode19@gmail.com}

\author{Shicheng Wang}
\authornotemark[1]
\affiliation{%
  \institution{Kuaishou Technology}
  \city{Beijing}
  \country{China}
  }
\email{wangshicheng@kuaishou.com}

\author{Shikun Li}
\authornote{Corresponding author.}
\affiliation{%
  \institution{Hong Kong Baptist University}
  \city{Hong Kong}
  \country{China}
  }
\email{shikunli.ml@gmail.com}

\author{Rui Huang}
\affiliation{%
  \institution{Kuaishou Technology}
  \city{Beijing}
  \country{China}
  }
\email{huangrui06@kuaishou.com}

\author{Xinghua Zhang}
\affiliation{%
\institution{Independent Researcher}
  \city{Beijing}
  \country{China}
  }
\email{zxh.zhangxinghua@gmail.com}

\author{Yunze Luo}
\affiliation{%
  \institution{Peking University}
  \city{Beijing}
  \country{China}
  }
\email{lyztangent@pku.edu.cn}

\author{Shipeng Li}
\affiliation{%
  \institution{Nanjing University}
  \city{Suzhou}
  \country{China}
  }
\email{shipengli.nju@gmail.com}


\author{Shiming Ge}
\affiliation{%
  \institution{Institute of Information Engineering, Chinese Academy of Sciences}
  \city{Beijing}
  \country{China}
  }
\email{geshiming@iie.ac.cn}

\author{Ruina Sun}
\affiliation{%
  \institution{Kuaishou Technology}
  \city{Beijing}
  \country{China}
  }
\email{sunruina@kuaishou.com}

\author{Yinjie Jiang}
\affiliation{%
  \institution{Kuaishou Technology}
  \city{Beijing}
  \country{China}
  }
\email{jiangyinjie@kuaishou.com}

\author{Jun Zhang}
\affiliation{%
  \institution{Kuaishou Technology}
  \city{Beijing}
  \country{China}
  }
\email{zhangjun08@kuaishou.com}

\renewcommand{\shortauthors}{Gaode Chen et al.}

\begin{abstract}

Implicit feedback is widely used in recommender systems due to its accessibility and generality, yet it usually presents noisy samples (e.g., clickbait, position bias). Meanwhile, recommenders inevitably face the item cold-start problem due to the continuous influx of new items. We identify that cold items are more prone to noisy samples due to the aforementioned factors, and researchers often overlook the significance of denoising implicit feedback for cold items. Previous denoising studies usually identify noisy samples based on heuristic patterns, such as higher loss values, and mitigate noise through sample selection or re-weighting. However, these methods have limited adaptability and are ineffective in cold-start scenarios. To achieve denoising implicit feedback for cold-start recommendation, we propose a model-agnostic denoising method called DIF. First, user preferences for content remain stable, which allows us to infer pseudo-labels indicating whether a user is interested in a cold item through content-similar warm items. We also elaborate on how to deploy industrial services to retrieve content-similar warm items for the cold item and obtain their collaborative representations for pseudo-labeling. Furthermore, to improve pseudo-label accuracy, we model the confidence of pseudo-labels based on the content similarity between the cold item and warm items, and then aggregate multiple pseudo-labels for each sample. Finally, we explicitly estimate the uncertainty of the noisy sample label by considering its relative entropy and the cold-start status of the item, which adaptively guides the role of pseudo-labels to correct the noisy labels at the sample level. DIF's superiority is supported by both theoretical justification and extensive experiments on real-world datasets. The method has been deployed on a billion-user scale short video application Kuaishou and has significantly improved various commercial metrics within cold-start scenarios.
\end{abstract}



\maketitle

\section{Introduction}

In the era of information explosion, recommender systems have been playing a critical role for mitigating information overload in various online applications such as short video~\cite{gong2022real,cai2023two} and E-commerce~\cite{zhou2018deep,pi2020search}. 
As the most successful technique for personalized recommender systems, collaborative filtering aims to predict items of interest to a specific user based on observed user–item interactions~\cite{ijcai2021p197,zhang2019deep}. 
While explicit user feedback (e.g., ratings) is the best fuel for recommender systems, its acquisition is often impeded by the need for active user participation. 
Hence, implicit feedback (e.g.,  view and click) generated during user browsing is exploited as a viable substitute due to its large volume and availability. 

However, the item cold-start problem~\cite{cao2022gift,zhou2023contrastive} may occur when the recommendation model faces a large volume of new items published daily. 
For example, tens of millions of new videos are uploaded every day on Kuaishou, one of the most popular short-video streaming platforms in China with hundreds of millions of active users. 
A small portion of popular items tend to obtain more accurate recommendations and more impressions, which creates a strong feedback loop and the Matthew effect, namely ``rich gets richer''~\cite{wang2023fresh}. 
Concretely, for a large amount of new emerging items with limited interactions, their embeddings are insufficiently trained, resulting in new items that may miss the opportunity to be recommended or be recommended to inappropriate users. 
Thus, the cold-start problem has become a crucial obstacle for online recommendation. 

As is well known, industrial recommendation systems inevitably contain noisy implicit feedback~\cite{zhao2024denoising,gao2022self}. 
However, in this work, we highlight an overlooked fact that cold-start items are more prone to having noisy implicit feedback.
We demonstrate various factors that generate noisy feedback on the Kuaishou platform to support our claim. 
On the one hand, users may be influenced by the clickbait issue, where curiosity about the cover or title of a short video leads to clicks or views, resulting in \textit{false positive samples}. 
Benefiting from the increasingly mature multi-modal content understanding in the industry, we find that cold-start videos in the top 10\% scoring range of clickbait metrics exceed those in the bottom 10\% scoring range by 37.7\%. 
On the other hand, users might overlook certain items due to position bias or scroll too fast during browsing fatigue, resulting in a lack of exposure to items and leading to \textit{false negative samples}. 
We analyze 1 million user sessions and observe that the proportion of cold-start videos placed in the tail three positions is 28.3\% higher than their proportion in the top three positions.
Hence, the implicit feedback on these short videos may not always indicate the actual satisfaction of users.  
In other words, cold-start videos are more likely to exhibit label noise due to these factors.

However, most existing work fails to recognize the importance of denoising implicit feedback for cold-start items. 
Industrial recommendation systems primarily focus on target user exploration for items in the cold-start phase, as these items lack well-trained representations. 
This often results in fewer positive feedback samples and more negative feedback samples.
Positive feedback serves as a critical signal for optimizing the representation of cold-start items in recommendation models. 
It guides these models to subsequently recommend these cold-start items to more suitable users, helping them accumulate more positive feedback and exposure.
Regardless of whether the false positive or false negative samples, it misleads recommendation models and wastes trial-and-error costs, as exposure opportunities are both limited and valuable.  
Such scenarios risk creating a feedback loop of inaccurate learning, which exacerbates the misrepresentation of cold-start items.
This not only leads to a decline in user experience but also undermines the motivation of producers, potentially causing a missed opportunity for an item to become a breakout success.
As such, denoising implicit feedback for cold-start recommendation becomes an imperative task. 

Considering the widespread use of implicit feedback and its significant impact on recommendation models, recent studies have noticed the importance of denoising implicit feedback. 
Existing efforts on tackling this problem can be roughly divided into two categories: sample selection methods and sample re-weighting methods.
Sample selection methods~\cite{wang2021denoising,gao2022self} focus on improving model performance by selecting clean samples and discarding noisy ones during training.
In contrast, sample re-weighting methods~\cite{hu2021next,wang2021denoising} aim to assign lower weights to interactions identified as noisy, thereby reducing their influence on the model’s learning process.
The success of these denoising techniques largely depends on the accuracy of distinguishing between clean and noisy samples. 
Consequently, various data patterns have been explored as noisy signals. 
For example, loss value is one of the most commonly used signals, as noisy interactions tend to exhibit higher loss values compared to clean ones~\cite{he2024double,wang2021denoising}. 
Additionally, other metrics, such as predicted scores~\cite{wang2022learning} and gradients~\cite{wang2023efficient} have also been investigated for identifying noisy samples. 
However, cold-start items typically inherently exhibit larger loss values, smaller predicted scores, and higher gradients in the recommendation model, 
making these denoising methods ineffective in cold-start scenarios.
Moreover, these denoising methods face challenges in integrating effectively with industrial recommendation models and applying them to online streaming training scenarios. 
As a result, existing denoising methods cannot be directly applied to industrial cold-start recommendation tasks.

To achieve the \textbf{D}enoising \textbf{I}mplicit \textbf{F}eedback for the cold-start recommendation, we propose a model-agnostic method \textbf{DIF}.
Designing such a method still faces many unknowns: 
(i) \textit{How to design a reasonable pseudo-labeling strategy for cold-start items?}
User interests in the content are generally stable. 
Thus, whether a user interacts with a cold-start item can be reasonably inferred through the warm items that are content-similar to the cold item. 
The interest representation of user and the collaborative representations of warm items are well-trained and can provide meaningful pseudo-labels. 
Moreover, we theoretically prove that even if content similarity between items cannot be fully guaranteed, pseudo-labels can still approximate the true label as long as the number of warm items is sufficient.
(ii) \textit{How to generate pseudo-labels for cold-start items in online streaming training?}
Retrieving content-similar warm items and their collaborative representations in industrial systems remains challenging due to the need for comprehensive candidate sets and real-time updates. 
We detail our practice experiences in applying our denoising method to online recommendation models.
(iii) \textit{How to further improve the accuracy of pseudo-labels?}
We not only consider the top-$k$ warm items most similar to the cold-start item to generate multiple pseudo-labels for aggregation, but also explicitly model the confidence of each pseudo-label based on content similarity during the pseudo-label aggregation process. 
(iv) \textit{How to correct the noisy sample label based on the final pseudo-label?}
We model the uncertainty of the sample label via relative entropy and cold-start status to adaptively guide the role of the pseudo-label during the label correction process.

To summarize, the key contributions are as follows:
\begin{itemize}[leftmargin=*]
\item We highlight that cold-start items are especially vulnerable to noisy implicit feedback, underscoring the need for denoising.
\item We propose a theoretically grounded method using multi-modal semantic similarity to generate pseudo-labels for cold items, along with practical deployment insights.
\item We incorporate confidence modeling in the pseudo-label aggregation process for higher accuracy, and explicitly measure label uncertainty during correction to adaptively control sample-level pseudo-label impact.
\item Extensive offline experiments on three real-world datasets validate the effectiveness of  DIF and its generalizability across different tasks. 
DIF has been deployed on a billion-user scale short video application Kuaishou, yielding significant improvement on a series of commercial metrics in cold-start scenarios.
\end{itemize}

\section{Methodology}
\subsection{Preliminary}

We give a formal description of the denoising implicit feedback for the cold-start recommendation task. 
Let $u \in \mathcal{U}$ and $i \in \mathcal{I}$ denote users and items, with the observed implicit feedback matrix $\tilde{\mathbf{Y}} \in \mathbb{R}^{|\mathcal{U}| \times |\mathcal{I}|}$.
And $\tilde{y}_{ui}$ $=$ $1$ means that the user interacted with the item, and $\tilde{y}_{ui}$ $=$ $0$ means no interaction. 
In previous work, the default assumption is that whenever $\tilde{y}_{ui}$ $=$ $1$, it means that the user is interested in the item. 
However, user implicit feedback may contain noise due to various factors (i.e., clickbait or position bias), resulting in noisy sample labels (\textit{false positive} or \textit{false negative}). 
Thus, our goal is to learn the noise-free representations of users and items from $\tilde{\mathbf{Y}}$, specifically for cold-start items.
In our application, a new or cold-start item is defined as a short video released in less than 24 hours (inclusive) and viewed less than 50,000 times. 

Industrial recommenders usually follow a multi-stage cascade architecture~\cite{zhang2023divide, gong2022real} (such as \textit{retrieval}~\cite{yang2020mixed}, \textit{pre-ranking} and \textit{ranking}~\cite{zhou2018deep}) to trade off accuracy and efficiency.
In theory, our method works for all the stages and is model-agnostic. 
In practice, we implement and deploy it based on the dual-tower model in the retrieval stage.
Dual-tower is the mainstream structure used by the industrial recommendation system in the retrieval stage.
Retrieval methods use a user tower and an item tower to transform user-side and item-side features into embeddings $\mathbf{e}_u$, $\mathbf{e}_i$ respectively.
The similarity between the user embedding $\mathbf{e}_u$ and the item embedding $\mathbf{e}_i$
is used to capture the relevance of user $u$ to item $i$ denoted by $\hat{y}_{ui}$.
During prediction, item embeddings ${\left\{\mathbf{e}_i\right\}}_{i \in \mathcal{I}}$ are calculated
beforehand in an offline or near-line system and indexed by the approximate nearest neighbor (ANN) search system, such as FAISS~\citep{johnson2019billion}. 
Thus, only the user tower needs to be forwarded in real-time and the ANN system can efficiently retrieve the nearest items in sub-linear time.

\subsection{Overview}

\begin{figure}[t]
\centering
\includegraphics[width=\linewidth]{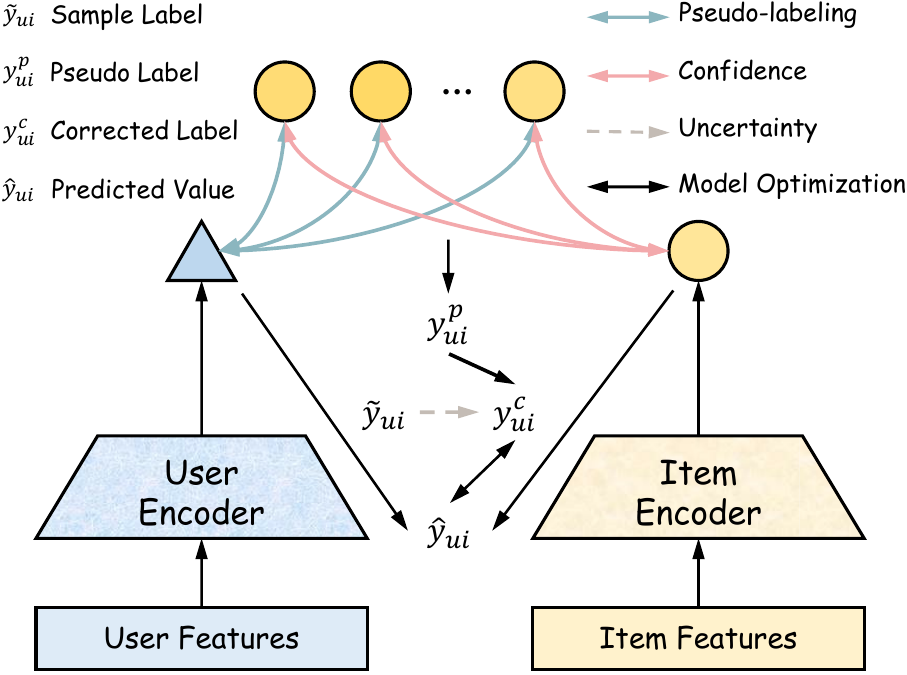}
\caption{The overview of our proposed DIF integrated with a dual-tower model.}
\label{method}
\end{figure}

\begin{figure*}[t]
\centering
\includegraphics[width=\linewidth]{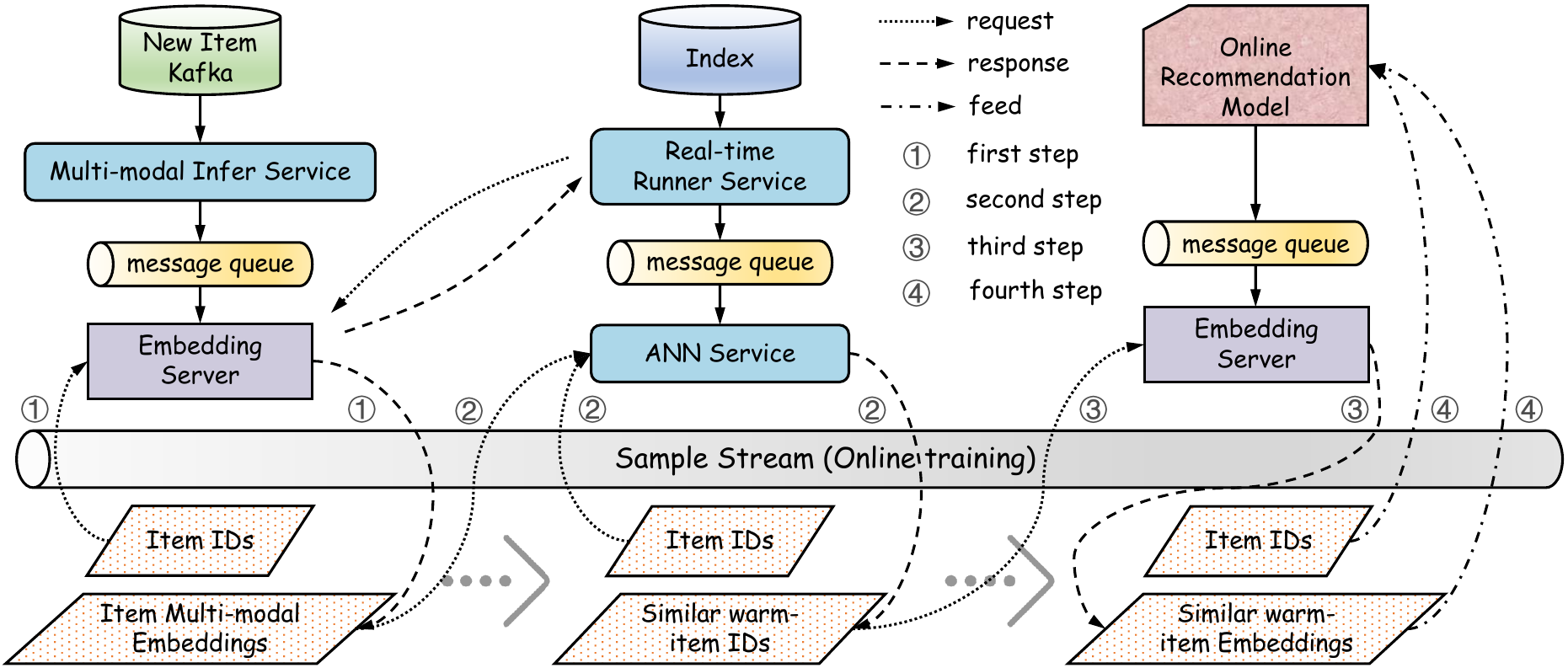}
\caption{The industrial practice for pseudo-labeling in online streaming training.}
\label{framework}
\end{figure*}

In Figure~\ref{method}, we present an overview of the proposed DIF integrated with a base recommendation model.
First, existing denoising methods in general cases typically perform pseudo-labeling based on the predictions from model~\cite{tanaka2018joint,arazo2019unsupervised}. 
However, due to representation issues, the predictions for cold-start items are often inaccurate. 
In contrast, the collaborative representations of some warm items are well-trained. 
Thus, we propose a novel denoising strategy for cold-start recommendation: generating pseudo-labels for the cold item based on the collaborative representations of content-similar warm items and user interest representation. 
This involves establishing services for retrieving content-similar warm items for the cold item and obtaining real-time collaborative representations of warm items. 
We also provide detailed practical experience.
Second, by retrieving the top-$k$ content-similar warm items for a cold item, we obtain $k$ pseudo-labels. 
We further model the confidence of pseudo-labels based on the content similarity, 
and then aggregate multiple pseudo-labels for each sample to improve accuracy.
Finally, we estimate sample label uncertainty via relative entropy and cold-start status to adaptively perform sample-level label correction based on the final pseudo-label.

\subsection{Pseudo-labeling}
\subsubsection{Practical Experience}
As shown in Figure~\ref{framework}, 
We address two key practical questions in our industrial environment: how to retrieve content-similar warm items for the cold item in sample stream and how to obtain collaborative representations of these warm items for subsequent pseudo-labeling.

First, for each newly uploaded item $i$ (short video), we can obtain its multi-modal content representation through the Multi-modal Inference Service and send it to the message queue. 
We have an Embedding Service that reads the content representation of new items from the message queue at all times and stores it.
For each item in the data stream, we can request the Embedding Service to query the content representation of the item (\textbf{step $\textcircled{1}$}). 

Moreover, we build an \textit{item-to-item} ANN service, where left \textit{items} are our queries and right \textit{items} are the warm items bucket we have constructed as candidates for retrieval.
To ensure sufficient coverage of warm item candidates, we utilize a Real-time Runner Service to continuously retrieve eligible warm items (e.g., videos with more than 50,000 views) from the Index, and request the Embedding Service in the previous step to obtain content representations of warm items, which are sent together to the message queue for continuous updating of ANN service candidates.
For the item $i$ and its content representations in the data stream, we feed them as input to request the ANN service (\textbf{step $\textcircled{2}$}), which quickly retrieves the top-$k$ IDs of warm items with content similar to item $i$, along with their similarity scores (i.e., $S_i$ $=$ $\left[s^1_i, s^2_i, ..., s^k_i\right]$) relative to it.

Furthermore, the online recommendation model actively sends the collaborative representations of users and items to a message queue. 
We establish an Embedding Service to store the collaborative representations of all items. 
Since an item (rather than a sample) may appear multiple times in the sample stream, the online recommendation model continuously updates its representation and sends it to the message queue to update the embedding storage in real time.
We feed the warm items IDs to query this Embedding Service, obtaining their collaborative representations $\mathbf{E}_i$ $=$ $\left[\mathbf{e}^1_i, \mathbf{e}^2_i, ..., \mathbf{e}^k_i\right]$ for subsequent pseudo-label generation (\textbf{step $\textcircled{3}$}). 

At this point, we obtain the collaborative representations $\mathbf{E}_i$ of the top-$k$ warm items with content similar to each item $i$ in the data stream, which can be used for pseudo-labeling in the online recommendation model (\textbf{step $\textcircled{4}$}). 

\subsubsection{Generation of Pseudo Labels}

The user interest and the collaborative representations of warm items can be considered well-trained.
Therefore, predicting interactions between the user and warm items provides valuable reference for predicting whether the user will interact with the content-similar cold-start item. 
For a sample $\left(u,i,\tilde{y}_{ui}\right)$, a dual-tower model typically predicts the interaction probability between user $u$ and item $i$ by computing the inner product between $\mathbf{e}_u$ and $\mathbf{e}_i$. 
Thus, to ensure homogeneity between pseudo-labels and the sample label, we also take the inner product to generate pseudo-labels. 
For $k$ warm items that have similar content to item $i$ in a sample, we obtain $k$ pseudo-labels as follows: 
\begin{align}
    \left[y_{ui,1}^{p}, y_{ui,2}^{p}, ..., y_{ui,k}^{p}\right] = \mathbf{e}_u \cdot \left[\mathbf{e}^1_i, \mathbf{e}^2_i, ..., \mathbf{e}^k_i\right]\ ,
\label{eq-pseudo-label}
\end{align}
where $\cdot$ denotes inner product.

\subsection{Confidence Modeling}
Although we later theoretically show that our method can approximate the true labels without strong similarity assumptions between cold and warm items, different similarity levels should still reflect different degrees of guidance strength.
Thus, it is intuitive to model the confidence of each pseudo-label based on the similarity between the target item and warm items. 
As shown below, after obtaining normalized confidence scores $c_i$, we compute a weighted sum to derive a more accurate pseudo-label $y_{ui}^{p}$: 

\begin{align}
&c_i^j = \frac{{\rm exp}\left(s_i^j / \tau\right)}{\sum_{z = 1}^{k} {\rm exp}\left(s_i^z / \tau\right)}, j = 1, 2, ..., k\ ,
\\
&y_{ui}^{p} = \sum_{j = 1}^{k} c_i^j  y_{ui,j}^{p}
\label{eq-confidence}\ .
\end{align}
Here, $\tau$ is a temperature coefficient. Benefiting from the existing advanced multi-modal content understanding capabilities, online observations indicate that similarity scores are generally above 0.9, and $\tau$ serves to enhance differentiation.

\subsection{Uncertainty Estimation}

To effectively perform label correction, we estimate‌ the uncertainty of the noisy sample label $\tilde{y}_{ui}$ from two perspectives to adaptively guide the role of pseudo-labels $y_{ui}^{p}$ at the sample level.

First, we consider the relative entropy~\cite{cover1991entropy} of the sample label with the predicted value (e.g., $\hat{y}_{ui}$ $=$ $\left[\mathbf{e}_u\right]^\top \cdot \mathbf{e}_i$) as a crucial indicator of the uncertainty, termed as $r_{ui}$. 
The higher the relative entropy, the more likely the sample is mislabeled. Its calculation is as below:

\begin{align}
r_{ui} = -\tilde{y}_{ui} \log \hat{y}_{ui} - \left(1 - \tilde{y}_{ui}\right) \log \left(1 - \hat{y}_{ui}\right)
\label{eq-cross-entropy}\ .
\end{align}

Second, as previously mentioned, cold-start items are more prone to generating label noise due to various factors. 
Thus, we explicitly model the cold-start state $\gamma_i$ of an item $i$ as another indicator for the uncertainty estimation:
\begin{align}
\gamma_i = e^{- \alpha  g_i}\ ,
\label{eq-cold-start}
\end{align}
where $e$ denotes the natural constant, 
$g_i$ represents the interaction count of the item $i$, and $\alpha$ controls the threshold for defining the cold-start state. 
A higher interaction count drives $\gamma_i$ closer to 0, while fewer interactions push it closer to 1. 
This design effectively prevents any adverse impact on the training of warm items.

Then, by integrating relative entropy $r_{ui}$ and cold-start status $\gamma_i$, we assign higher uncertainty to cold-start items with significant prediction discrepancies. 
The final uncertainty estimation $t_{ui}$ is formulated as shown below and constrained to be less than 1.
\begin{align}
t_{ui} = r_{ui} \gamma_i \ .
\label{eq-ue}
\end{align}

\subsection{Model Optimization}

To alleviate the impact of label noise, we adaptively correct the sample label based on its uncertainty at the sample level. 
Specifically, taking the weight of the original sample label as 1, we normalize the uncertainty by scaling it within the range by Eq.~(\ref{eq-weight}), and then we obtain 
the corrected soft label $y_{ui}^{c}$ for each sample by Eq.~(\ref{eq-revision}). 
\begin{align}
& w = \frac{t_{ui}}{t_{ui} + 1}\ ,
\label{eq-weight}
\\
& y_{ui}^{c} = w  y_{ui}^{p} + (1 - w)  \tilde{y}_{ui}\ .
\label{eq-revision}
\end{align}

Finally, we utilize the widely-used cross-entropy loss for training:
\begin{align}
\mathcal{L} = -\sum_{(u,i) \in \mathcal{D}_{\text{train}}}\left(y_{ui}^{c} \log \hat{y}_{ui} + \left(1 - y_{ui}^{c}\right) \log \left(1 - \hat{y}_{ui}\right)\right)\ .
\label{eq-loss}
\end{align}

\subsection{Deployment}

Our proposed DIF is trained on a large-scale distributed nearline learning system of Kuaishou. 
Every day, hundreds of millions of users visit Kuaishou, watch and interact with short videos, and yield tens of billions watch and interaction logs per day. 
On the one hand, each log is collected in real-time, preprocessed by the Kafka stream processing platform, and produces the sample stream. 
In Section 2.3, we provide a detailed practice experience of how to obtain the collaborative representations of content-similar warm items for each item in the sample stream. 
As training data, including the aforementioned features, is received by the downstream recommendation model through the pipeline, this near-line training system incrementally updates the model parameters using the latest knowledge from user-video interactions.
In streaming training, our approach effectively minimizes the impact on the training of warm items by leveraging uncertainty estimation while focusing on assisting cold-start items in correcting noisy labels during the early stages, thereby enabling higher growth potential. 

\subsection{Theoretical Justification}
In this section, we present the theoretical analyses to justify the validity of our denoising approach. 

Let $e_i \in \mathcal{E} \subset \mathbb{R}^d$ be the normalized multi-modal embedding of a cold-start item $i$. 
Let $\mathcal{N}_k(i)$  be the set of $k$ nearest warm items retrieved based on context similarity. 
Let $\eta^*_{ui}=\eta^*_{u}(\mathbf{e}_i)$ be the ground-truth interaction probability between user $u$ and item $i$.  Let ${y}^p_{ui,j}$ be the pseudo label between user $u$ and warm neighbor item $j \in \mathcal{N}_k(i)$.  Let ${y}^p_{ui}$ be the confidence-weighed pseudo label between user $u$ and cold-start item $i$. 
\newtheorem{Assumption}{Assumption}
\numberwithin{Assumption}{section}

\begin{Assumption}
The observed cold-start item label  $\tilde{y}_{ui}$ is a biased estimator of its true interaction probability $\eta^*_{ui}$, i.e., $\tilde{y}_{ui} = \eta^*_{ui} + \epsilon_{ui}$, where $\mathbb{E}[\epsilon_{ui}] =  b_{ui}$ and $\text{Var}[\epsilon_{ui}] = \sigma^2_{ui}$.
    \label{assum1}
\end{Assumption}
	{Remark~1.} Assumption~\ref{assum1} means that due to various factors, such as clickbait issues and position bias, the observed cold-start item label has undeniable  systematic bias noise $b_{ui}$ and relatively large noise variance  $\sigma^2_{ui}$.

    
\begin{Assumption}
The pseudo warm item label  ${y}^p_{ui,j}$ is an unbiased estimator of  its true interaction probability $\eta^*_{uj}$, i.e., ${y}^p_{ui,j} = \eta^*_{uj} + \epsilon'_{ui,j}$, where $\mathbb{E}[\epsilon'_{ui,j}] =  0$ and $\text{Var}[\epsilon'_{ui,j}] = \sigma'^2_{u}$.
    \label{assum2}
	\end{Assumption}
    
	{Remark~2.} Since the user interest representation and collaborative representations of warm items can be considered well-trained with a sufficiently large sample, Assumption~\ref{assum2} posits  that the pseudo warm item label has negligible approximation error ($\mathbb{E}[\epsilon'_{ui,j}]$) and small variance $\sigma'^2_{u}$.
    
  \begin{Assumption} The context representation space $\mathcal{E}$ is a compact subset of $\mathbb{R}^d$, and the probability density function $p(\mathbf{e})$ is bounded away from zero on $\mathcal{E}$, i.e., there exists a constant $p_{\inf} > 0$ such that $p(\mathbf{e}) \geq p_{\inf}$ for all $\mathbf{e} \in \mathcal{E}$.
\label{assum3}
\end{Assumption}
  \begin{Assumption} The true interaction probability function $\eta^*_{u}(\mathbf{e})$ is $\alpha$-Hölder continuous on the content representations space, i.e., there exists $0<\alpha<1$ and $L>0$ such that $|\eta_u^*(\mathbf{e}_i) - \eta_u^*(\mathbf{e}_j)| \leq L \|\mathbf{e}_i - \mathbf{e}_j\|^\alpha$, for any $\mathbf{e}_i,\mathbf{e}_j \in \mathcal{E}$.
    \label{assum4}
	\end{Assumption}
   {Remark~3.} Assumption~\ref{assum3} ensures valid nearest neighbors exist within a small radius around any item when having a sufficiently large sample, and Assumption~\ref{assum4} makes a smoothness assumption on $\eta^*_{u}(\mathbf{e})$. Both of them are standard conditions for the analysis under k-NN classification~\cite{bahri2020deep,gao2018consistency,gyorfi2002distribution}.
    
	\begin{theorem} 
    Suppose Assumptions~\ref{assum2}, ~\ref{assum3}, and~\ref{assum4} hold.  For any cold-start item $i \in \mathcal{I}$, user $u$ $\in$ $\mathcal{U}$, $\delta\in (0,1)$, 
    we have
		\label{thm1}
			$$|y^p_{ui} - \eta^*_{ui}| \leq \mathcal{O}\left( L \left(\frac{k}{N_{warm}}\right)^{\frac{\alpha}{d}} \right) + \mathcal{O}\left(\frac{\sigma'_{u}}{\sqrt{k}}\right),$$
            where $k$ is the number of context-similar neighbors, $N_{warm}$ is the number of warm items, $d$ is the dimension of the context space $\mathcal{E}$.
	\end{theorem}
\begin{proof}
Under Assumption~\ref{assum2}, we first decompose the error as follows:
\begin{equation}
\begin{aligned}
|y^p_{ui} - \eta^*_{ui}| &=\left| \sum_{j \in \mathcal{N}_k(i)} c_{i}^j \tilde{y}_{ui,j} - \eta^*_{ui} \right| = \left| \sum_{j \in \mathcal{N}_k(i)} c_{i}^j (\tilde{y}_{ui,j} - \eta^*_{ui}) \right| \\
&= \left| \sum_{j \in \mathcal{N}_k(i)} c_{i}^j (\eta^*_{uj} + \epsilon'_{ui,j} - \eta^*_{ui}) \right| \\
&\leq {\left| \sum_{j \in \mathcal{N}_k(i)} c_{i}^j (\eta^*_{uj} - \eta^*_{ui}) \right|} + {\left| \sum_{j \in \mathcal{N}_k(i)} c_{i}^j \epsilon'_{ui,j} \right|}.
\end{aligned}
\label{eq1}
\end{equation}

For the first term ${\left| \sum_{j \in \mathcal{N}_k(i)} c_{i}^j (\eta^*_{uj} - \eta^*_{ui}) \right|}$, it can be further decomposed under Assumption~\ref{assum4}:
\begin{equation}
\left| \sum_{j \in \mathcal{N}_k(i)} c_{i}^j (\eta^*_{uj} - \eta^*_{ui}) \right| \leq \sum_{j \in \mathcal{N}_k(i)} c_{i}^j L \|\mathbf{e}_j - \mathbf{e}_i\|^\alpha \leq L \cdot \max_{j \in \mathcal{N}_k(i)} \|\mathbf{e}_j - \mathbf{e}_i\|^\alpha.
\end{equation}
According to non-parametric regression theory (See Theorem 6.2 in~\cite{gyorfi2002distribution}), in a $d$-dimensional space with $N_{warm}$ samples and a probability density function bounded away from zero (Assumption~\ref{assum3}), the distance between a target point and its $k$-th nearest neighbor $\max_{j \in \mathcal{N}_k(i)} \|\mathbf{e}_j - \mathbf{e}_i\|$ is bounded by $\mathcal{O}((k/N_{warm})^{1/d})$. Substituting this into the above inequality, we obtain: 
\begin{equation}
\left| \sum_{j \in \mathcal{N}_k(i)} c_{i}^j (\eta^*_{uj} - \eta^*_{ui}) \right| \leq \mathcal{O}\left( L \left(\frac{k}{N_{warm}}\right)^{\frac{\alpha}{d}} \right).
\label{eq2}
\end{equation}

For the second term ${\left| \sum_{j \in \mathcal{N}_k(i)} c_{i}^j \epsilon'_{ui,j} \right|}$, by assuming $\epsilon'_{ui,j}$ represents bounded independent zero-mean random noise with variance (Assumption~\ref{assum2}), by applying Hoeffding's Inequality~\cite{Boucheron2013Concentration}, we have:
\begin{equation}
{\left| \sum_{j \in \mathcal{N}_k(i)} c_{i}^j \epsilon'_{ui,j} \right|} \leq \mathcal{O}\left(\frac{\sigma'_{u}}{\sqrt{k_{eff}}}\right) \approx \mathcal{O}\left(\frac{\sigma'_{u}}{\sqrt{k}}\right),
\label{eq3}
\end{equation}
where $k_{eff} = 1 / \left(\sum \left(c_{i}^j\right)^2\right)$, and in the high-dimensional spaces with dense sampling, $k_{eff} \approx k$ due to the concentration of measure. 

Finally, substituting Eq.~(\ref{eq2}) and Eq.~(\ref{eq3}) into Eq.~(\ref{eq1}), we can obtain:
\begin{equation}
\left|y^p_{ui} - \eta^*_{ui}\right| \leq \mathcal{O}\left( L \left(\frac{k}{N_{warm}}\right)^{\frac{\alpha}{d}} \right) + \mathcal{O}\left(\frac{\sigma'_{u}}{\sqrt{k}}\right)
\end{equation}
\end{proof}

{Remark~4.} Theorem~\ref{thm1} clearly shows that, when the number of warm items $N_{warm}$ is  sufficiently large, by choosing a suitable $k$, the error bound between $y^p_{ui}$ and $ \eta^*_{ui}$ will become small, which clearly justifies the cleanness of the pseudo labels $y^p_{ui}$.

\begin{theorem} Suppose Assumptions~\ref{assum1} and~\ref{assum2} hold.  Assume the observation noise $\epsilon_{ui}$ and the pseudo-label noise $\epsilon'_{ui,j}$ are uncorrelated. There exists an optimal weight $w^* \in [0, 1]$ that minimizes the MSE risk $R(y^c_{ui}) = \mathbb{E}[(y^c_{ui} - \eta^*_{ui})^2]$. Furthermore, the optimal weight is $$w^* = \frac{\lambda_{ui}}{\lambda_{ui} + 1},\text{and}~ \lambda_{ui} = \frac{R(\tilde{y}_{ui})}{R(y^p_{ui})},$$ 
where $\lambda_{ui}$ represents the noise ratio of $\tilde{y}_{ui}$ to $y^p_{ui}$.
\label{thm2}
\end{theorem}
\begin{proof}
To minimize the Mean Squared Error (MSE) of the corrected label $y^c_{ui}$ with respect to the true probability $\eta^*_{ui}$:
\begin{equation}
\begin{aligned}
R(y^c_{ui}) & = \mathbb{E}\left[ \left( w y^p_{ui} + (1-w) \tilde{y}_{ui} - \eta^*_{ui} \right)^2 \right] \\
&=  w^2 \mathbb{E}[(y^p_{ui} - \eta^*_{ui})^2] + (1-w)^2 \mathbb{E}[(\tilde{y}_{ui} - \eta^*_{ui})^2] \\
&+ 2w(1-w) \mathbb{E}[(y^p_{ui} - \eta^*_{ui})(\tilde{y}_{ui} - \eta^*_{ui})] \\
&=  w^2 R(y^p_{ui})+ (1-w)^2 R(\tilde{y}_{ui}),
\end{aligned}
\end{equation}
where the last equation holds because we assume the observation noise and the pseudo-label noise  are uncorrelated.
Taking the derivative with respect to $w$ and setting it to zero:
$$2w R(y^p_{ui}) - 2(1-w) R(\tilde{y}_{ui}) = 0.$$
Solving for the optimal weight:
$$w^* = \frac{R(\tilde{y}_{ui})}{R(\tilde{y}_{ui}) + R(y^p_{ui})} = \frac{\lambda_{ui}}{\lambda_{ui}+1}.$$
\end{proof}
{Remark~5.} Theorem~\ref{thm2} derives an optimal weight $w^* = \frac{\lambda_{ui}}{\lambda_{ui} + 1}$, which mathematically aligns with the weighting scheme proposed in Eq.~(\ref{eq-weight}). Since $R(\tilde{y}_{ui})$ and $R(y^p_{ui})$ are unknown, we utilize the uncertainty term $t_{ui} $ to measure the observation noise level by the relative entropy $r_{ui}$ and cold-start status $\gamma_i$. It can be regarded as an effective estimator of the noise ratio for denoising.

\section{Experiments}
\subsection{Experimental Settings}
\subsubsection{Dataset}
\begin{table}[t]
\caption{Statistics of experimented datasets with multi-modal item Visual(V), Acoustic(A), Textual(T) contents.}
\begin{tabular}{cccccccc}
\toprule
\multirow{2}{*}{Dataset} 
    & \multicolumn{4}{c}{Amazon} 
    & \multicolumn{3}{c}{\multirow{2}{*}{Tiktok}} \\ 
\cmidrule{2-5}
    & \multicolumn{2}{c}{Sports} 
    & \multicolumn{2}{c}{Baby} 
    & \multicolumn{3}{c}{} \\ 
\midrule
Modality 
    & V              & T           
    & V              & T           
    & V              & A           
    & T \\
Embed Dim                
    & 4096           & 1024        
    & 4096           & 1024        
    & 128            & 128         
    & 768 \\ 
\midrule
User                     
    & \multicolumn{2}{c}{35,598}   
    & \multicolumn{2}{c}{19,445}   
    & \multicolumn{3}{c}{9,319} \\
Item                     
    & \multicolumn{2}{c}{18,357}   
    & \multicolumn{2}{c}{7,050}    
    & \multicolumn{3}{c}{6,710} \\
Interactions             
    & \multicolumn{2}{c}{256,308}  
    & \multicolumn{2}{c}{139,110}  
    & \multicolumn{3}{c}{59,541} \\ 
\midrule
Sparsity                 
    & \multicolumn{2}{c}{99.961\%} 
    & \multicolumn{2}{c}{99.899\%} 
    & \multicolumn{3}{c}{99.904\%} \\ 
\bottomrule
\end{tabular}
\label{datasets}
\end{table}

Following~\cite{chen2024multi}, we conduct offline experiments on three real-world multi-modal recommendation datasets~\footnote{All datasets are publicly available at \url{https://github.com/HKUDS/MMSSL/tree/main}.}, i.e., Amazon-Sports, Amazon-Baby, Tiktok. 
Data statistics with multi-modal feature embedding dimensionality are reported in Table \ref{datasets}.
\begin{itemize}[leftmargin=*]
\item \textbf{Amazon.} We adopt two benchmark datasets from Amazon with two item categories Baby and Sports. 
In those datasets, textual feature embeddings are generated via Sentence-Bert~\cite{reimers2019sentence} based on the extracted text from the product title, description, brand, and categorical information. 
The product images are used to generate 4096-$d$ visual feature embeddings of items.
\item \textbf{TikTok.} This data is collected from TikTok platform to log the viewed short-videos of users. 
The multi-modal features are visual, acoustic, and title textual features of videos. 
The textual embeddings are also encoded with Sentence-Bert. 
\end{itemize}

\subsubsection{Evaluation Protocols}

For each dataset, we used the ratio 8:1:1 to randomly split the historical interactions of each user and constituted the training set, validation set, and testing set. 
Moreover, following the widely-used evaluation metrics~\cite{ijcai2021p197, wei2023multi}, we adopted Precision@N, Recall@N, and NDCG@N to evaluate the performance of methods. 
By default, we set N = 20 and reported the average values of the three metrics for all users in the testing set.

\begin{table*}[t]
\centering
\caption{Performance comparison of baselines on different datasets in terms of Recall@20 and NDCG@20. The best performance is highlighted in \textbf{bold} and the second to best is highlighted by \underline{underlines}.}
\label{main_performance}
\begin{tabular}{c|c|c|cc|cc|cc}
\toprule
\multicolumn{1}{c|}{\multirow{2}{*}{}} & \multicolumn{1}{c|}{\multirow{2}{*}{Base Model}} & \multicolumn{1}{c|}{\multirow{2}{*}{Method}} & \multicolumn{2}{c|}{Amazon-Sports} & \multicolumn{2}{c|}{Amazon-Baby} & \multicolumn{2}{c}{Tiktok} \\
& \multicolumn{1}{c|}{} & \multicolumn{1}{c|}{} & Recall@20 & \multicolumn{1}{c|}{NDCG@20} & Recall@20 & \multicolumn{1}{c|}{NDCG@20} & Recall@20 & \multicolumn{1}{c}{NDCG@20} \\ \midrule

\multicolumn{1}{c|}{\multirow{15}{*}{Cold}} & \multicolumn{1}{c|}{\multirow{5}{*}{NeuMF}} & Normal & 0.00244 & 0.00228 & 0.00142 & 0.00157 & 0.00104 & 0.00010 \\
& \multicolumn{1}{c|}{} & RINCE & 0.00271 & 0.00222 & \underline{0.00197} & \underline{0.00172} & \underline{0.00674} & \textbf{0.00298} \\
& \multicolumn{1}{c|}{} & DECL & 0.00266 & 0.00191 & 0.00186 & 0.00149 & 0.00622 & \underline{0.00234} \\
& \multicolumn{1}{c|}{} & MWUF & \underline{0.00302} & \underline{0.00250} & 0.00166 & 0.00158 & 0.00052 & 0.00013 \\
& \multicolumn{1}{c|}{} & Ours & \textbf{0.00461} & \textbf{0.00259} & \textbf{0.00385} & \textbf{0.00188} & \textbf{0.00881} & 0.00227 \\ \cmidrule{2-9}
& \multicolumn{1}{c|}{\multirow{5}{*}{LightGCN}} & Normal & 0.00191 & 0.00192 & 0.00143 & 0.00122 & 0.00106 & 0.00026 \\
& \multicolumn{1}{c|}{} & RINCE & 0.00266 & 0.00294 & \underline{0.00258} & \underline{0.00248} & \underline{0.00415} & \textbf{0.00125} \\
& \multicolumn{1}{c|}{} & DECL & \underline{0.00277} & \underline{0.00306} & 0.00179 & 0.00168 & 0.00231 & 0.00065 \\
& \multicolumn{1}{c|}{} & MWUF & 0.00253 & 0.00222 & 0.00186 & 0.00161 & 0.00155 & 0.00050 \\
& \multicolumn{1}{c|}{} & Ours & \textbf{0.00325} & \textbf{0.00341} & \textbf{0.00303} & \textbf{0.00261} & \textbf{0.00518} & \underline{0.00117} \\ \cmidrule{2-9}
& \multicolumn{1}{c|}{\multirow{5}{*}{SimGCL}} & Normal & 0.00218 & 0.00210 & 0.00180 & 0.00184 & 0.00155 & 0.00040 \\
& \multicolumn{1}{c|}{} & RINCE & 0.00293 & 0.00312 & \underline{0.00285} & \underline{0.00289} & \underline{0.00466} & \underline{0.00101} \\
& \multicolumn{1}{c|}{} & DECL & 0.00317 & \underline{0.00361} & 0.00219 & 0.00184 & 0.00259 & 0.00072 \\
& \multicolumn{1}{c|}{} & MWUF & \underline{0.00319} & 0.00270 & 0.00236 & 0.00227 & 0.00207 & 0.00060 \\
& \multicolumn{1}{c|}{} & Ours & \textbf{0.00347} & \textbf{0.00379} & \textbf{0.00350} & \textbf{0.00356} & \textbf{0.00570} & \textbf{0.00154} \\ \midrule

\multicolumn{1}{c|}{\multirow{15}{*}{Warm}} & \multicolumn{1}{c|}{\multirow{5}{*}{NeuMF}} & Normal & 0.04535 & 0.02134 & 0.05534 & 0.02456 & 0.10095 & 0.04033 \\
& \multicolumn{1}{c|}{} & RINCE & \textbf{0.06619} & \underline{0.02755} & 0.05435 & 0.02220 & 0.09690 & 0.03936 \\
& \multicolumn{1}{c|}{} & DECL & 0.06073 & 0.02595 & 0.05487 & 0.02236 & 0.08667 & 0.03324 \\
& \multicolumn{1}{c|}{} & MWUF & 0.04594 & 0.02058 & \underline{0.05768} & \underline{0.02489} & \textbf{0.11119} & \textbf{0.04685} \\
& \multicolumn{1}{c|}{} & Ours & \underline{0.06281} & \textbf{0.02848} & \textbf{0.05980} & \textbf{0.02832} & \underline{0.11095} & \underline{0.04440} \\ \cmidrule{2-9}
& \multicolumn{1}{c|}{\multirow{5}{*}{LightGCN}} & Normal & 0.06381 & 0.02769 & 0.05882 & 0.02438 & 0.09429 & 0.04429 \\
& \multicolumn{1}{c|}{} & RINCE & 0.09695 & 0.04314 & \underline{0.08852} & \underline{0.03838} & \underline{0.12119} & \textbf{0.05403} \\
& \multicolumn{1}{c|}{} & DECL & \underline{0.09937} & \underline{0.04436} & 0.06250 & 0.02632 & 0.11500 & 0.04423 \\
& \multicolumn{1}{c|}{} & MWUF & 0.06953 & 0.03029 & 0.06316 & 0.02696 & 0.11024 & 0.04769 \\
& \multicolumn{1}{c|}{} & Ours & \textbf{0.10862} & \textbf{0.04869} & \textbf{0.09494} & \textbf{0.04257} & \textbf{0.12667} & \underline{0.05134} \\ \cmidrule{2-9}
& \multicolumn{1}{c|}{\multirow{5}{*}{SimGCL}} & Normal & 0.06550 & 0.02972 & 0.05934 & 0.02546 & 0.09571 & 0.04785 \\
& \multicolumn{1}{c|}{} & RINCE & 0.09815 & 0.04444 & \underline{0.09316} & \underline{0.04132} & 0.12167 & 0.05183 \\
& \multicolumn{1}{c|}{} & DECL & \underline{0.10749} & \underline{0.04943} & 0.06177 & 0.02609 & 0.09833 & 0.04649 \\
& \multicolumn{1}{c|}{} & MWUF & 0.07235 & 0.03227 & 0.06189 & 0.02614 & \underline{0.12524} & \underline{0.05732} \\
& \multicolumn{1}{c|}{} & Ours & \textbf{0.11410} & \textbf{0.05312} & \textbf{0.10098} & \textbf{0.04427} & \textbf{0.12881} & \textbf{0.05758} \\ \midrule

\multicolumn{1}{c|}{\multirow{15}{*}{Overall}} & \multicolumn{1}{c|}{\multirow{5}{*}{NeuMF}} & Normal & 0.02932 & 0.01375 & 0.04261 & 0.01891 & 0.06949 & 0.02763 \\
& \multicolumn{1}{c|}{} & RINCE & \textbf{0.04270} & \underline{0.01776} & 0.04191 & 0.01710 & 0.06852 & 0.02791 \\
& \multicolumn{1}{c|}{} & DECL & 0.03922 & 0.01674 & 0.04231 & 0.01723 & 0.06134 & 0.02351 \\
& \multicolumn{1}{c|}{} & MWUF & 0.02987 & 0.01335 & \underline{0.04442} & \underline{0.01917} & \underline{0.07635} & \textbf{0.03214} \\
& \multicolumn{1}{c|}{} & Ours & \underline{0.04114} & \textbf{0.01829} & \textbf{0.04674} & \textbf{0.02182} & \textbf{0.07879} & \underline{0.03113} \\ \cmidrule{2-9}
& \multicolumn{1}{c|}{\multirow{5}{*}{LightGCN}} & Normal & 0.04073 & 0.01768 & 0.04524 & 0.01875 & 0.06493 & 0.03043 \\
& \multicolumn{1}{c|}{} & RINCE & 0.06189 & 0.02754 & \underline{0.06809} & \underline{0.02952} & \underline{0.08434}& \textbf{0.03741} \\
& \multicolumn{1}{c|}{} & DECL & \underline{0.06344} & \underline{0.02832} & 0.04807 & 0.02024 & 0.07912 & 0.03039 \\
& \multicolumn{1}{c|}{} & MWUF & 0.04445 & 0.01935 & 0.04858 & 0.02073 & 0.07602 & 0.03283 \\
& \multicolumn{1}{c|}{} & Ours & \textbf{0.06934} & \textbf{0.03109} & \textbf{0.07303} & \textbf{0.03275} & \textbf{0.08842} & \underline{0.03555} \\ \cmidrule{2-9}
& \multicolumn{1}{c|}{\multirow{5}{*}{SimGCL}} & Normal & 0.04207 & 0.01904 & 0.04564 & 0.01958 & 0.06607 & 0.03291\\
& \multicolumn{1}{c|}{} & RINCE & 0.06266 & 0.02837 & \underline{0.07165} & \underline{0.03178} & 0.08483 & 0.03583 \\
& \multicolumn{1}{c|}{} & DECL & \underline{0.06862} & \underline{0.03156} & 0.04751 & 0.02007 & 0.06819 & 0.03208 \\
& \multicolumn{1}{c|}{} & MWUF & 0.04641 & 0.02066 & 0.04760 & 0.02011 & \underline{0.08646} & \underline{0.03946}\\
& \multicolumn{1}{c|}{} & Ours & \textbf{0.07287} & \textbf{0.03392} & \textbf{0.07767} & \textbf{0.03405} & \textbf{0.09005} & \textbf{0.03993} \\ \bottomrule
\end{tabular}
\end{table*}

\subsubsection{Baseline Methods}

To demonstrate the efficacy of our proposed DIF in denoising implicit feedback for cold-start recommendation, we compare DIF with the state-of-the-art model-agnostic denoising and cold-start methods. 
In particular, 
1) \textbf{DECL}~\cite{qin2022deep} captures the uncertainty caused by noise and separates clean and noisy data, employing them differently for learning with noisy correspondence. 
2) \textbf{RINCE}~\cite{chuang2022robust} introduces a new contrastive learning objective designed to be robust to noisy data views. 
3) \textbf{MWUF}~\cite{zhu2021learning} proposes meta-scaling and shifting networks to generate initial embedding representations for cold-start items within the warm feature space, demonstrating greater stability from noisy samples.  
We implement DIF and the aforementioned model-agnostic baselines to three representative backend models. 
1) \textbf{NeuMF}~\cite{he2017neural} models the relationship between users and
items by combining GMF and a MLP. 
2) \textbf{LightGCN}~\cite{he2020lightgcn} leverages high-order neighbors information to enhance the user and item representations. 
3) \textbf{SimGCL}~\cite{yu2022graph} simplifies the contrastive loss and optimizes uniformity representations.

\subsubsection{Implementation Details}

We set the embedding dimension $d$ fixed to 64 for all models. 
We optimize all models with the Adam~\cite{kingma2014adam} optimizer, where the batch size is fixed at 1024. 
We set the learning rate to 0.001, the number of similar warm items $k$ to 5, and the temperature coefficient $\tau$ to 0.3. 
We use the Xavier initializer\cite{glorot2010understanding}, and distinguish warm items and cold items based on the 20th percentile of interaction counts for each dataset and determine the specific values of $\alpha$ and $\gamma$ in uncertainty estimation.

\subsection{Performance Comparison}

We conduct comprehensive experiments in natural noise setting to compare DIF’s performance with other referenced baselines. 
The results, including on cold, warm, and overall items, are illustrated in Table~\ref{main_performance}, yields several key observations:
\begin{itemize}[leftmargin=*]
\item All the denoising methods show better results compared with normal training, especially in the improvement of performance on cold items, which indicates the necessity of denoising implicit feedback for cold-start recommendations. 
\item Jointly analyzing the performance of different backbones in recommender systems, we observe that NeuMF demonstrates more significant performance on cold items compared to graph-based methods such as LightGCN and SimGCL. 
This advantage arises because graph-based methods rely heavily on the propagation of collaborative signals along edges, making them more effective in data-rich scenarios. 
Nevertheless, our method is model-agnostic and achieves the best results across various backbones on cold items, which demonstrates the superior scalability of DIF.
\item DECL and RINCE, as general denoising methods directly transferred to recommendation tasks, improve performance on both cold and warm items, though the gains are not substantial. 
MWUF performs reasonably well on cold items but shows limited improvement on warm items and overall performance. 
In contrast, our method focuses on denoising implicit feedback for cold items without causing unacceptable impacts on warm items or overall performance, even achieving performance enhancements.
\item Our DIF does not achieve the optimal NDCG on cold items for Tiktok. 
This is partly due to the small and sparser nature of the TikTok. 
Moreover, our primary focus is on retrieve optimization rather than ranking, which remains one of future directions.
\item The proposed DIF effectively enhances the performance of all base models and outperforms most denoising methods across three datasets. 
We attribute these improvements to DIF's content-based denoising approach using warm items for cold items: 
(1) By modeling the confidence of multiple pseudo-labels, DIF further improves the accuracy of aggregated pseudo-labels based on content similarity scores. 
(2) Benefiting from our uncertainty estimation, DIF accurately utilizes pseudo-labels and guides the sample label correction process. 
\end{itemize}

\subsection{Ablation Study}

\begin{table}[t]
\caption{Ablation study on key components of DIF.}
\begin{tabular}{clccc}
\toprule
Datasets                & Method & Cold & Warm & Overall \\ \midrule
\multirow{4}{*}{Tiktok} & $\mathtt{w/o}$ CM & 0.00311 & 0.10857 & 0.07537 \\
 & $\mathtt{w/o}$ RE & 0.00152 & 0.12214 & 0.08385 \\
 & $\mathtt{w/o}$ CS & 0.00104 & 0.10619 & 0.07308 \\
 & DIF & \textbf{0.00518} & \textbf{0.12667} & \textbf{0.08842} \\ \midrule
\multirow{4}{*}{Sports} & $\mathtt{w/o}$ CM & 0.00293 & 0.09804 & 0.06259 \\
 & $\mathtt{w/o}$ RE & 0.00278 & 0.09685 & 0.06182 \\
 & $\mathtt{w/o}$ CS  & 0.00264 & 0.09318 & 0.05954\\
 & DIF & \textbf{0.00325} & \textbf{0.10862} & \textbf{0.06934} \\ \bottomrule
\end{tabular}
\label{ablation}
\end{table}

To evaluate the effectiveness of each component in our method, we perform ablation studies on Tiktok and Amazon Sports using LightGCN as the backbone.
We present the results of Recall@20 in Table~\ref{ablation}. We can see that:
\begin{itemize}[leftmargin=*]
\item Without the Confidence Modeling ($\mathtt{w/o}$ CM), we aggregate multiple pseudo-labels for each sample using average pooling, without considering content similarity distinctions. 
This variant has minimal impact on the performance of cold items, suggesting that current multi-modal representations are already highly reliable, with well-guaranteed similarity among the retrieved warm items.
\item We ablate the relative entropy component from the uncertainty estimation with the variant  $\mathtt{w/o}$ RE, which significantly affects cold items but has a limited impact on warm items. 
High relative entropy indicates a greater likelihood of label noise. 
However, since the collaborative representations of warm items are already well-trained, situations with high relative entropy are rare.
\item We make another comparison between DIF and the variant ($\mathtt{w/o}$ CS) without consideration of cold-start status, which demonstrates the most significant impact for cold items. 
This can be attributed to two factors: First, the failure to account for the varying probabilities of label noise across different cold-start states. Second, the potential disruption to representation learning for warm items, which may subsequently affect user modeling and ultimately influence recommendations for cold items.
\end{itemize}

\subsection{Noise Robustness}

\begin{figure*}[t]
\centering
\includegraphics[width=\linewidth]{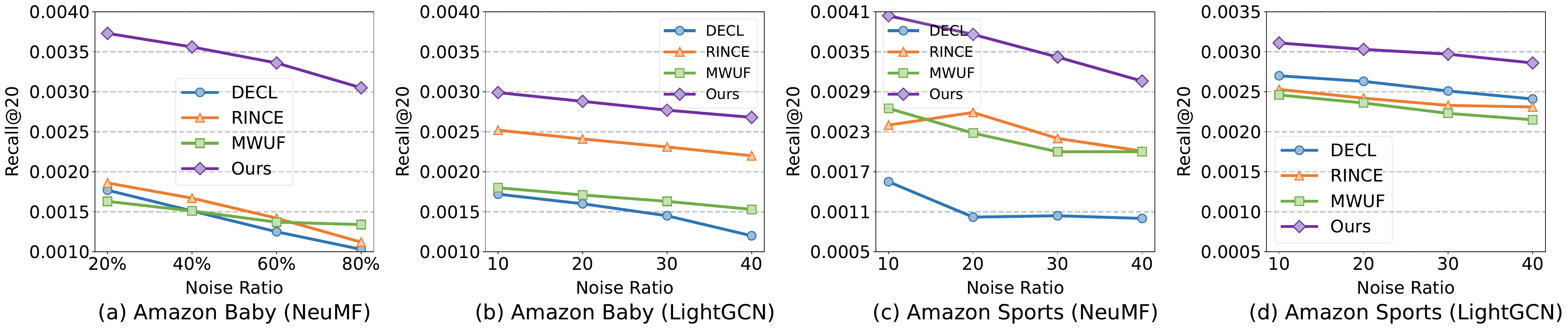}
\caption{Performance comparison of denoise training with random noises in Amazon datasets.}
\label{robust}
\end{figure*}

We conduct random noisy training on Amazon datasets using the backbones of NeuMF and LightGCN to evaluate the noise resistance capability of DIF, comparing it with the three competitive model-agnostic methods, DECL, RINCE and MWUF. 
The proportion of noise in our training settings spanned from 20\% to 80\%. 
We report the results on cold items in Figure~\ref{robust}. 
Similar results are seen with Tiktok and other backbones, but figures are omitted for brevity. 
The results show that: 
1) As the noise ratio increases, the performance of DECL, RINCE, MWUF, and DIF generally declines. 
This decline is attributed to the intensifying corruption of data due to the escalating noise level, making the representation learning for cold-start items and the modeling of user preferences increasingly challenging. 
2) DIF consistently outperforms all other denoising methods across different noise ratio settings. 
This highlights the remarkable noise robustness of DIF in cold-start scenarios, which can be attributed to the effectiveness of its strategy of generating pseudo-labels through content-similar warm items.


\subsection{Online A/B Testing}

\begin{table}[t]
\caption{The results of A/B testing in online scenario.}
\begin{tabular}{cccc}
\toprule
Effective View & Watch Time & Long View & Short View \\ 
\rowcolor{mygray}+2.327\% &+2.921\% &+2.688\% & -2.030\% \\ \midrule
Like & Comment & Follow & Share \\
\rowcolor{mygray}+2.921\% &+2.435\% &+2.790\% &+0.876\% \\ \bottomrule
\end{tabular}
\label{ab}
\end{table}

We carried out rigorous online A/B testing in our short-video streaming scenario from Jan. 12, 2025, to Jan. 16, 2025, with hundreds of millions of users per day. 
The results of the online A/B test for cold-start recommendation are shown in Table~\ref{ab}, we focus on several commercial metrics, which can be divided into view-related and action-related. 
For view-related metrics, we introduce metrics such as \textit{effective view}, \textit{watch time}, \textit{long view}, and \textit{short view}. 
We define an ``effective view'' label as user watches a video longer than a threshold (e.g., 5 seconds), and videos in different duration intervals have different thresholds.
Meanwhile, the ``long view'' label follows stricter criterias compared to ``effective view''.
For action-related metrics, include \textit{like}, \textit{comment}, \textit{follow}, and \textit{share}. 
For example, a ``like'' is defined as the user likes current video by clicking the like button or double tapping/long pressing screen. 
For company privacy, we cannot report the real metrics of the original online models. 
Instead, we report the performance gain ratio improved by our approach DIF. 
It is worth noting that one percent improvement ratio usually indicates a large improvement of the recommendation capacity in real-world application scenario, when tested on a large population of users. 
By effectively denoising cold-start item samples using our method, both view-related and action-related metrics show significant improvement, contributing to a substantial enhancement in the recommendation efficiency for cold-start items.

\section{Related Works}
\textbf{Cold-start Recommendation.}
Cold-start is one of the main challenges in recommender systems. 
Specifically, we focus on the more challenging scenario where cold-start items are newly uploaded and lack any user feedback. 
The common solution to this issue can be categorized into two types, namely content-based and transfer learning based methods.
The first type of methods~\cite{wei2023multi,chen2024multi,wang2024large,jeon2024cold,liu2024fine} aims to exploit side information, especially content features of items to compensate the absence of collaborative signals. 
The core idea is to approximate the well-trained collaborative embeddings via content information by modeling their correlation.  
Another way to alleviate the cold-start problem is to transfer knowledge from other domains, such as cross-domain recommendation~\cite{cao2022cross,zang2022survey,chen2023win}, transfer learning methods~\cite{zhang2021model,sheng2021one,zhang2023empowering}, meta-learning methods~\cite{dong2020mamo,zhu2021learning}, and prompt learning~\cite{jiang2024prompt,kusano2024data}.
However, few existing studies on cold-start recommendation recognize the importance of implicit feedback denoising and design practical deployment solutions.

\noindent
\textbf{Denoising Recommendation.}
Existing recommender systems are typically trained with implicit feedback. 
Recently, some studies~\cite{wang2021denoising,wang2021implicit,zhao2024denoising,wang2023efficient} have noticed that implicit feedback could be easily corrupted by different factors (e.g., popularity bias~\cite{chen2023bias} and user unconscious behaviors~\cite{hu2008collaborative}), and the inevitable noise would dramatically degrade the recommendation performance. 
As a result, some efforts have been dedicated to solving the noisy implicit feedback problem, which can be categorized into sample selection methods~\cite{ding2019sampler,park2019adversarial,yu2020sampler,LearnAlign} and sample re-weighting methods~\cite{hu2021next,wang2022learning}.
The sample selection methods aim to select clean and informative samples, which exhibit high performance variation since they heavily depend on the sampling distribution.
Sample re-weighting methods focus on assigning lower weights to inaccurate interactions by the learning process signals of models (e.g., loss values and predictions). 
However, these methods heavily rely on predefined heuristic assumptions, resulting in poor generalization for cold-start models.

\section{Conclusion}
To address the implicit feedback denoising in cold-start recommendation, this paper introduces a model-agnostic approach called DIF, which can be flexibly applied to various online models. 
Inspired by the item collaborative filtering, we assign pseudo-labels to cold-start items through the user and warm items with similar content, and we present a detailed industrial implement solution. 
We further model the confidence of each pseudo-label based on content similarity to improve the accuracy of the final generated pseudo label.
Furthermore, we estimate the uncertainty of the noisy sample label to guide the participation of the pseudo-label during the label correction process, thereby avoiding overcorrection.
We share the successful experience of deploying DIF on Kuaishou. Extensive offline experiments and online A/B testing further verify its effectiveness in cold-start recommendation tasks. 

\bibliographystyle{ACM-Reference-Format}
\bibliography{sample-base}

@book{Boucheron2013Concentration,
title={Concentration Inequalities - A Nonasymptotic Theory of Independence},
	author={St{\'e}phane Boucheron and G{\'a}bor Lugosi and Pascal Massart},
    publisher= {OUP Oxford},
year={2013}
}

@inproceedings{LearnAlign,
  title={LearnAlign: Data selection for LLM reinforcement learning with improved gradient alignment},
  author={Shipeng Li and Zhiqin Yang and Shikun Li and Xiaobo Xia and Hengyu Liu and Xinghua Zhang and Gaode Chen and Dong Fang and Ying Tai and Zhe Peng},
  booktitle={Findings of the Association for Computational Linguistics: ACL 2026},
  year ={2026}
}

@inproceedings{bahri2020deep,
  title={Deep k-nn for noisy labels},
  author={Bahri, Dara and Jiang, Heinrich and Gupta, Maya},
  booktitle={International Conference on Machine Learning},
  pages={540--550},
  year={2020},
  organization={PMLR}
}

@article{gao2018consistency,
  title={On the consistency of exact and approximate nearest neighbor with noisy data},
  author={Gao, Wei and Niu, Xin-Yi and Zhou, Zhi-Hua},
  journal={Arxiv, abs/1607.07526},
  year={2018}
}

@book{gyorfi2002distribution,
  title={A distribution-free theory of nonparametric regression},
  author={Gy{\"o}rfi, L{\'a}szl{\'o} and Kohler, Michael and Krzy{\.z}ak, Adam and Walk, Harro},
  year={2002},
  publisher={Springer}
}

@inproceedings{gong2022real,
  title={Real-time short video recommendation on mobile devices},
  author={Gong, Xudong and Feng, Qinlin and Zhang, Yuan and Qin, Jiangling and Ding, Weijie and Li, Biao and Jiang, Peng and Gai, Kun},
  booktitle={Proceedings of the 31st ACM international conference on information \& knowledge management},
  pages={3103--3112},
  year={2022}
}

@article{cover1991entropy,
  title={Entropy, relative entropy and mutual information},
  author={Cover, Thomas M and Thomas, Joy A and others},
  journal={Elements of information theory},
  volume={2},
  number={1},
  pages={12--13},
  year={1991}
}

@inproceedings{tanaka2018joint,
  title={Joint optimization framework for learning with noisy labels},
  author={Tanaka, Daiki and Ikami, Daiki and Yamasaki, Toshihiko and Aizawa, Kiyoharu},
  booktitle={Proceedings of the IEEE conference on computer vision and pattern recognition},
  pages={5552--5560},
  year={2018}
}

@inproceedings{arazo2019unsupervised,
  title={Unsupervised label noise modeling and loss correction},
  author={Arazo, Eric and Ortego, Diego and Albert, Paul and O’Connor, Noel and McGuinness, Kevin},
  booktitle={International conference on machine learning},
  pages={312--321},
  year={2019},
}

@inproceedings{cai2023two,
  title={Two-stage constrained actor-critic for short video recommendation},
  author={Cai, Qingpeng and Xue, Zhenghai and Zhang, Chi and Xue, Wanqi and Liu, Shuchang and Zhan, Ruohan and Wang, Xueliang and Zuo, Tianyou and Xie, Wentao and Zheng, Dong and others},
  booktitle={Proceedings of the ACM Web Conference 2023},
  pages={865--875},
  year={2023}
}

@inproceedings{zhou2018deep,
  title={Deep interest network for click-through rate prediction},
  author={Zhou, Guorui and Zhu, Xiaoqiang and Song, Chenru and Fan, Ying and Zhu, Han and Ma, Xiao and Yan, Yanghui and Jin, Junqi and Li, Han and Gai, Kun},
  booktitle={Proceedings of the 24th ACM SIGKDD international conference on knowledge discovery \& data mining},
  pages={1059--1068},
  year={2018}
}

@inproceedings{pi2020search,
  title={Search-based user interest modeling with lifelong sequential behavior data for click-through rate prediction},
  author={Pi, Qi and Zhou, Guorui and Zhang, Yujing and Wang, Zhe and Ren, Lejian and Fan, Ying and Zhu, Xiaoqiang and Gai, Kun},
  booktitle={Proceedings of the 29th ACM International Conference on Information \& Knowledge Management},
  pages={2685--2692},
  year={2020}
}

@inproceedings{ijcai2021p197,
  title     = {Exploring Periodicity and Interactivity in Multi-Interest Framework for Sequential Recommendation},
  author    = {Chen, Gaode and Zhang, Xinghua and Zhao, Yanyan and Xue, Cong and Xiang, Ji},
  booktitle = {Proceedings of the Thirtieth International Joint Conference on Artificial Intelligence},
  pages     = {1426--1433},
  year      = {2021},
  month     = {8},
}

@inproceedings{chen2024multi,
  title={A Multi-modal Modeling Framework for Cold-start Short-video Recommendation},
  author={Chen, Gaode and Sun, Ruina and Jiang, Yuezihan and Cao, Jiangxia and Zhang, Qi and Lin, Jingjian and Li, Han and Gai, Kun and Zhang, Xinghua},
  booktitle={Proceedings of the 18th ACM Conference on Recommender Systems},
  pages={391--400},
  year={2024}
}

@inproceedings{jiang2024prompt,
  title={Prompt Tuning for Item Cold-start Recommendation},
  author={Jiang, Yuezihan and Chen, Gaode and Zhang, Wenhan and Wang, Jingchi and Jiang, Yinjie and Zhang, Qi and Lin, Jingjian and Jiang, Peng and Bian, Kaigui},
  booktitle={Proceedings of the 18th ACM Conference on Recommender Systems},
  pages={411--421},
  year={2024}
}

@inproceedings{cao2022gift,
  title={Gift: Graph-guided feature transfer for cold-start video click-through rate prediction},
  author={Cao, Yi and Hu, Sihao and Gong, Yu and Li, Zhao and Yang, Yazheng and Liu, Qingwen and Ji, Shouling},
  booktitle={Proceedings of the 31st ACM International Conference on Information \& Knowledge Management},
  pages={2964--2973},
  year={2022}
}

@inproceedings{zhou2023contrastive,
  title={Contrastive Collaborative Filtering for Cold-Start Item Recommendation},
  author={Zhou, Zhihui and Zhang, Lilin and Yang, Ning},
  booktitle={Proceedings of the ACM Web Conference 2023},
  pages={928--937},
  year={2023}
}

@article{zhang2019deep,
  title={Deep learning based recommender system: A survey and new perspectives},
  author={Zhang, Shuai and Yao, Lina and Sun, Aixin and Tay, Yi},
  journal={ACM computing surveys (CSUR)},
  volume={52},
  number={1},
  pages={1--38},
  year={2019}
}

@article{zang2022survey,
  title={A survey on cross-domain recommendation: taxonomies, methods, and future directions},
  author={Zang, Tianzi and Zhu, Yanmin and Liu, Haobing and Zhang, Ruohan and Yu, Jiadi},
  journal={ACM Transactions on Information Systems},
  volume={41},
  number={2},
  pages={1--39},
  year={2022},
  publisher={ACM New York, NY}
}

@inproceedings{sheng2021one,
  title={One model to serve all: Star topology adaptive recommender for multi-domain ctr prediction},
  author={Sheng, Xiang-Rong and Zhao, Liqin and Zhou, Guorui and Ding, Xinyao and Dai, Binding and Luo, Qiang and Yang, Siran and Lv, Jingshan and Zhang, Chi and Deng, Hongbo and others},
  booktitle={Proceedings of the 30th ACM International Conference on Information \& Knowledge Management},
  pages={4104--4113},
  year={2021}
}

@inproceedings{dong2020mamo,
  title={Mamo: Memory-augmented meta-optimization for cold-start recommendation},
  author={Dong, Manqing and Yuan, Feng and Yao, Lina and Xu, Xiwei and Zhu, Liming},
  booktitle={Proceedings of the 26th ACM SIGKDD international conference on knowledge discovery \& data mining},
  pages={688--697},
  year={2020}
}

@inproceedings{zhu2021learning,
  title={Learning to warm up cold item embeddings for cold-start recommendation with meta scaling and shifting networks},
  author={Zhu, Yongchun and Xie, Ruobing and Zhuang, Fuzhen and Ge, Kaikai and Sun, Ying and Zhang, Xu and Lin, Leyu and Cao, Juan},
  booktitle={Proceedings of the 44th International ACM SIGIR Conference on Research and Development in Information Retrieval},
  pages={1167--1176},
  year={2021}
}

@inproceedings{reimers2019sentence,
  title={Sentence-BERT: Sentence Embeddings using Siamese BERT-Networks},
  author={Reimers, Nils and Gurevych, Iryna},
  booktitle={Proceedings of the 2019 Conference on Empirical Methods in Natural Language Processing and the 9th International Joint Conference on Natural Language Processing (EMNLP-IJCNLP)},
  pages={3982--3992},
  year={2019}
}

@inproceedings{wei2023multi,
  title={Multi-modal self-supervised learning for recommendation},
  author={Wei, Wei and Huang, Chao and Xia, Lianghao and Zhang, Chuxu},
  booktitle={Proceedings of the ACM Web Conference 2023},
  pages={790--800},
  year={2023}
}

@article{kingma2014adam,
  title={Adam: A method for stochastic optimization},
  author={Kingma, Diederik P and Ba, Jimmy},
  journal={arXiv preprint arXiv:1412.6980},
  year={2014}
}

@inproceedings{glorot2010understanding,
  title={Understanding the difficulty of training deep feedforward neural networks},
  author={Glorot, Xavier and Bengio, Yoshua},
  booktitle={Proceedings of the thirteenth international conference on artificial intelligence and statistics},
  pages={249--256},
  year={2010}
}

@inproceedings{he2017neural,
  title={Neural collaborative filtering},
  author={He, Xiangnan and Liao, Lizi and Zhang, Hanwang and Nie, Liqiang and Hu, Xia and Chua, Tat-Seng},
  booktitle={Proceedings of the 26th international conference on world wide web},
  pages={173--182},
  year={2017}
}

@inproceedings{he2020lightgcn,
  title={Lightgcn: Simplifying and powering graph convolution network for recommendation},
  author={He, Xiangnan and Deng, Kuan and Wang, Xiang and Li, Yan and Zhang, Yongdong and Wang, Meng},
  booktitle={Proceedings of the 43rd International ACM SIGIR conference on research and development in Information Retrieval},
  pages={639--648},
  year={2020}
}

@inproceedings{wang2023efficient,
  title={Efficient bi-level optimization for recommendation denoising},
  author={Wang, Zongwei and Gao, Min and Li, Wentao and Yu, Junliang and Guo, Linxin and Yin, Hongzhi},
  booktitle={Proceedings of the 29th ACM SIGKDD Conference on Knowledge Discovery and Data Mining},
  pages={2502--2511},
  year={2023}
}

@inproceedings{wang2022learning,
  title={Learning robust recommenders through cross-model agreement},
  author={Wang, Yu and Xin, Xin and Meng, Zaiqiao and Jose, Joemon M and Feng, Fuli and He, Xiangnan},
  booktitle={Proceedings of the ACM Web Conference 2022},
  pages={2015--2025},
  year={2022}
}

@inproceedings{wang2021implicit,
  title={Implicit feedbacks are not always favorable: Iterative relabeled one-class collaborative filtering against noisy interactions},
  author={Wang, Zitai and Xu, Qianqian and Yang, Zhiyong and Cao, Xiaochun and Huang, Qingming},
  booktitle={Proceedings of the 29th ACM International Conference on Multimedia},
  pages={3070--3078},
  year={2021}
}

@inproceedings{yu2020sampler,
  title={Sampler design for implicit feedback data by noisy-label robust learning},
  author={Yu, Wenhui and Qin, Zheng},
  booktitle={Proceedings of the 43rd international ACM SIGIR conference on research and development in information retrieval},
  pages={861--870},
  year={2020}
}

@inproceedings{hu2021next,
  title={What is next when sequential prediction meets implicitly hard interaction?},
  author={Hu, Kaixi and Li, Lin and Xie, Qing and Liu, Jianquan and Tao, Xiaohui},
  booktitle={Proceedings of the 30th ACM International Conference on Information \& Knowledge Management},
  pages={710--719},
  year={2021}
}

@inproceedings{wang2021denoising,
  title={Denoising implicit feedback for recommendation},
  author={Wang, Wenjie and Feng, Fuli and He, Xiangnan and Nie, Liqiang and Chua, Tat-Seng},
  booktitle={Proceedings of the 14th ACM international conference on web search and data mining},
  pages={373--381},
  year={2021}
}

@inproceedings{he2024double,
  title={Double correction framework for denoising recommendation},
  author={He, Zhuangzhuang and Wang, Yifan and Yang, Yonghui and Sun, Peijie and Wu, Le and Bai, Haoyue and Gong, Jinqi and Hong, Richang and Zhang, Min},
  booktitle={Proceedings of the 30th ACM SIGKDD Conference on Knowledge Discovery and Data Mining},
  pages={1062--1072},
  year={2024}
}

@inproceedings{wang2023fresh,
  title={Fresh Content Needs More Attention: Multi-funnel Fresh Content Recommendation},
  author={Wang, Jianling and Lu, Haokai and Zhang, Sai and Locanthi, Bart and Wang, Haoting and Greaves, Dylan and Lipshitz, Benjamin and Badam, Sriraj and Chi, Ed H and Goodrow, Cristos J and others},
  booktitle={Proceedings of the 29th ACM SIGKDD Conference on Knowledge Discovery and Data Mining},
  pages={5082--5091},
  year={2023}
}

@inproceedings{zhao2024denoising,
  title={Denoising diffusion recommender model},
  author={Zhao, Jujia and Wenjie, Wang and Xu, Yiyan and Sun, Teng and Feng, Fuli and Chua, Tat-Seng},
  booktitle={Proceedings of the 47th International ACM SIGIR Conference on Research and Development in Information Retrieval},
  pages={1370--1379},
  year={2024}
}

@inproceedings{gao2022self,
  title={Self-guided learning to denoise for robust recommendation},
  author={Gao, Yunjun and Du, Yuntao and Hu, Yujia and Chen, Lu and Zhu, Xinjun and Fang, Ziquan and Zheng, Baihua},
  booktitle={Proceedings of the 45th International ACM SIGIR Conference on Research and Development in Information Retrieval},
  pages={1412--1422},
  year={2022}
}

@inproceedings{zhang2023divide,
  title={Divide and Conquer: Towards Better Embedding-based Retrieval for Recommender Systems from a Multi-task Perspective},
  author={Zhang, Yuan and Dong, Xue and Ding, Weijie and Li, Biao and Jiang, Peng and Gai, Kun},
  booktitle={Companion Proceedings of the ACM Web Conference 2023},
  pages={366--370},
  year={2023}
}

@inproceedings{yang2020mixed,
  title={Mixed negative sampling for learning two-tower neural networks in recommendations},
  author={Yang, Ji and Yi, Xinyang and Zhiyuan Cheng, Derek and Hong, Lichan and Li, Yang and Xiaoming Wang, Simon and Xu, Taibai and Chi, Ed H},
  booktitle={Companion proceedings of the web conference 2020},
  pages={441--447},
  year={2020}
}

@article{johnson2019billion,
  title={Billion-scale similarity search with gpus},
  author={Johnson, Jeff and Douze, Matthijs and J{\'e}gou, Herv{\'e}},
  journal={IEEE Transactions on Big Data},
  volume={7},
  number={3},
  pages={535--547},
  year={2019},
  publisher={IEEE}
}

@inproceedings{chuang2022robust,
  title={Robust contrastive learning against noisy views},
  author={Chuang, Ching-Yao and Hjelm, R Devon and Wang, Xin and Vineet, Vibhav and Joshi, Neel and Torralba, Antonio and Jegelka, Stefanie and Song, Yale},
  booktitle={Proceedings of the IEEE/CVF Conference on Computer Vision and Pattern Recognition},
  pages={16670--16681},
  year={2022}
}

@inproceedings{yu2022graph,
  title={Are graph augmentations necessary? simple graph contrastive learning for recommendation},
  author={Yu, Junliang and Yin, Hongzhi and Xia, Xin and Chen, Tong and Cui, Lizhen and Nguyen, Quoc Viet Hung},
  booktitle={Proceedings of the 45th international ACM SIGIR conference on research and development in information retrieval},
  pages={1294--1303},
  year={2022}
}

@inproceedings{qin2022deep,
  title={Deep evidential learning with noisy correspondence for cross-modal retrieval},
  author={Qin, Yang and Peng, Dezhong and Peng, Xi and Wang, Xu and Hu, Peng},
  booktitle={Proceedings of the 30th ACM International Conference on Multimedia},
  pages={4948--4956},
  year={2022}
}

@article{chen2023bias,
  title={Bias and debias in recommender system: A survey and future directions},
  author={Chen, Jiawei and Dong, Hande and Wang, Xiang and Feng, Fuli and Wang, Meng and He, Xiangnan},
  journal={ACM Transactions on Information Systems},
  volume={41},
  number={3},
  pages={1--39},
  year={2023},
  publisher={ACM New York, NY}
}

@inproceedings{hu2008collaborative,
  title={Collaborative filtering for implicit feedback datasets},
  author={Hu, Yifan and Koren, Yehuda and Volinsky, Chris},
  booktitle={2008 Eighth IEEE international conference on data mining},
  pages={263--272},
  year={2008},
  organization={Ieee}
}

@article{ding2019sampler,
  title={Sampler design for bayesian personalized ranking by leveraging view data},
  author={Ding, Jingtao and Yu, Guanghui and He, Xiangnan and Feng, Fuli and Li, Yong and Jin, Depeng},
  journal={IEEE transactions on knowledge and data engineering},
  volume={33},
  number={2},
  pages={667--681},
  year={2019},
  publisher={IEEE}
}

@inproceedings{park2019adversarial,
  title={Adversarial sampling and training for semi-supervised information retrieval},
  author={Park, Dae Hoon and Chang, Yi},
  booktitle={The World Wide Web Conference},
  pages={1443--1453},
  year={2019}
}

@inproceedings{cao2022cross,
  title={Cross-domain recommendation to cold-start users via variational information bottleneck},
  author={Cao, Jiangxia and Sheng, Jiawei and Cong, Xin and Liu, Tingwen and Wang, Bin},
  booktitle={2022 IEEE 38th International Conference on Data Engineering (ICDE)},
  pages={2209--2223},
  year={2022},
  organization={IEEE}
}

@inproceedings{chen2023win,
  title={Win-win: a privacy-preserving federated framework for dual-target cross-domain recommendation},
  author={Chen, Gaode and Zhang, Xinghua and Su, Yijun and Lai, Yantong and Xiang, Ji and Zhang, Junbo and Zheng, Yu},
  booktitle={Proceedings of the AAAI Conference on Artificial Intelligence},
  volume={37},
  number={4},
  pages={4149--4156},
  year={2023}
}

@inproceedings{zhang2021model,
  title={A model of two tales: Dual transfer learning framework for improved long-tail item recommendation},
  author={Zhang, Yin and Cheng, Derek Zhiyuan and Yao, Tiansheng and Yi, Xinyang and Hong, Lichan and Chi, Ed H},
  booktitle={Proceedings of the web conference 2021},
  pages={2220--2231},
  year={2021}
}

@inproceedings{zhang2023empowering,
  title={Empowering Long-tail Item Recommendation through Cross Decoupling Network (CDN)},
  author={Zhang, Yin and Wang, Ruoxi and Cheng, Derek Zhiyuan and Yao, Tiansheng and Yi, Xinyang and Hong, Lichan and Caverlee, James and Chi, Ed H},
  booktitle={Proceedings of the 29th ACM SIGKDD Conference on Knowledge Discovery and Data Mining},
  pages={5608--5617},
  year={2023}
}

@inproceedings{kusano2024data,
  title={Data Augmentation using Reverse Prompt for Cost-Efficient Cold-Start Recommendation},
  author={Kusano, Genki},
  booktitle={Proceedings of the 18th ACM Conference on Recommender Systems},
  pages={861--865},
  year={2024}
}

@inproceedings{wang2024large,
  title={Large language models as data augmenters for cold-start item recommendation},
  author={Wang, Jianling and Lu, Haokai and Caverlee, James and Chi, Ed H and Chen, Minmin},
  booktitle={Companion Proceedings of the ACM Web Conference 2024},
  pages={726--729},
  year={2024}
}

@inproceedings{jeon2024cold,
  title={Cold-start bundle recommendation via popularity-based coalescence and curriculum heating},
  author={Jeon, Hyunsik and Lee, Jong-eun and Yun, Jeongin and Kang, U},
  booktitle={Proceedings of the ACM Web Conference 2024},
  pages={3277--3286},
  year={2024}
}

@article{liu2024fine,
  title={Fine Tuning Out-of-Vocabulary Item Recommendation with User Sequence Imagination},
  author={Liu, Ruochen and Chen, Hao and Bei, Yuanchen and Shen, Qijie and Zhong, Fangwei and Wang, Senzhang and Wang, Jianxin},
  journal={Advances in Neural Information Processing Systems},
  volume={37},
  pages={8930--8955},
  year={2024}
}


\end{document}